\pdfoutput=1
\documentclass[pdflatex,sn-mathphys-ay]{sn-jnl}

\usepackage{graphicx}%
\usepackage{multirow}%
\usepackage{amsmath,amssymb,amsfonts}%
\usepackage{amsthm}%
\usepackage{mathrsfs}%
\usepackage[title]{appendix}%
\usepackage{xcolor}%
\usepackage{textcomp}%
\usepackage{manyfoot}%
\usepackage{booktabs}%

\newcommand{\Sig}{\Sigma}
\newcommand{\Shat}{\widehat{\Sigma}}
\newcommand{\w}{w}
\newcommand{\what}{\widehat{w}}
\newcommand{\one}{\mathbf{1}}
\newcommand{\R}{\mathbb{R}}
\newcommand{\E}{\mathbb{E}}
\newcommand{\dw}{\mathrm{d}\w}
\newcommand{\regret}{\mathcal{R}}
\newcommand{\cond}{\operatorname{cond}}
\newcommand{\tr}{\operatorname{tr}}

\theoremstyle{thmstyleone}
\newtheorem{theorem}{Theorem}[section]
\newtheorem{proposition}[theorem]{Proposition}
\newtheorem{lemma}[theorem]{Lemma}

\theoremstyle{thmstylethree}

\theoremstyle{thmstyletwo}
\newtheorem{remark}{Remark}

\begin{document}

\title[Decision Geometry of Covariance Estimation for the GMVP]{The Decision Geometry of Covariance Estimation for the Global Minimum-Variance Portfolio under Heavy Tails}

\author*[1]{\fnm{Xavier} \sur{Fonseca}}\email{xavier.fonseca.phd@gmail.com}

\affil*[1]{\orgdiv{Academy for AI, Games and Media}, \orgname{Breda University of Applied Sciences}, \orgaddress{\city{Breda}, \country{The Netherlands}}}

\abstract{The global minimum-variance portfolio (GMVP) is the canonical decision built
from an estimated covariance matrix, yet covariance estimators are universally evaluated
by matrix-norm loss, which is not the object the decision depends on. We characterise
exactly how covariance-estimation error maps into GMVP suboptimality. We prove an exact
regret identity and a non-asymptotic bound showing decision regret depends on the
estimation error only through its action on the portfolio weights, scaled by portfolio
concentration and the conditioning of the true covariance. From this we derive the
decision geometry: GMVP regret is invariant to a $(p-1)$-dimensional projection of the
$p^2$-dimensional error matrix, with invariance to the covariance-scale direction as an
exact special case. We then apply the framework to heavy-tailed returns (tail index
$\kappa\in(2,4)$), establishing the regret convergence rate implied by the centred
operator-norm rate, and confirm the theory on a skew-$t$/$t$-copula simulation design
with pre-registered analysis. The decision-focused advantage is a sharper constant and a
concentration discount rather than a faster rate; we report an honest high-conditioning
boundary of the rate prediction. The results complement recent decision-focused learning
approaches by supplying the exact estimation geometry and consistency theory they lack.}

\keywords{Global minimum-variance portfolio, Covariance estimation, Decision-focused
learning, Regret bound, Heavy-tailed returns, Portfolio optimisation}

\pacs[JEL Classification]{C13, C58, G11}

\maketitle
\section{Introduction}\label{sec:intro}

Risk-based portfolio construction begins with an estimate of the covariance matrix of
asset returns, and the global minimum-variance portfolio (GMVP) is its purest instance:
it depends on the covariance matrix alone, with no expected-return input, and so isolates
the question of how covariance-estimation quality translates into portfolio quality.
That question is sharpened by a mismatch between what is estimated and what is measured.
The estimation literature evaluates covariance estimators by matrix-norm
loss---operator, Frobenius, or spectral---and a mature theory characterises minimax
rates for these losses, including dimension-free guarantees under heavy tails and
adversarial corruption \citep{abdalla2022} and robust estimators for heavy-tailed data
\citep{ke2019}. But matrix-norm loss is not the quantity the investor experiences. The
investor experiences \emph{regret}: the excess variance of the portfolio formed from the
estimate relative to the portfolio formed from the truth. A covariance estimator can be
poor in matrix norm yet induce small regret, or accurate in matrix norm yet induce large
regret, depending on which directions of the matrix it gets wrong relative to the
portfolio's exposures. The central question of this paper is the exact relationship
between covariance-estimation error and GMVP regret---which directions of error the
decision is exposed to, and how strongly.

Decision-focused learning (DFL) has recently been applied to exactly this gap, training
covariance models to optimise downstream decision quality rather than predictive error
\citep{kim2025,lee2024dfl}. These works demonstrate empirically that decision-aware
estimation can outperform error-minimising estimation for the GMVP, but they are
learning-based and offer no characterisation of \emph{why}, nor any consistency or
rate theory: they optimise against an oracle covariance and report out-of-sample
performance. This paper supplies the missing structure. We give the exact decision
geometry---which directions of covariance error the GMVP decision actually sees---and
the consistency theory under heavy tails.

\bmhead{Contributions} 
\begin{enumerate}
\item \textbf{Exact regret identity and non-asymptotic bound} (Theorem~\ref{thm:bound}).
For the GMVP, regret equals $(\what-\w)^{\!\top}\Sig(\what-\w)$ exactly, and is bounded
by $C_0\,\|\w\|^2\,\|E\|_2^2\,\cond(\Sig)$ with $E=\Shat-\Sig$ and $C_0$ an absolute
constant. The bound scales with \emph{portfolio concentration} $\|\w\|^2$ and the
conditioning of the truth, not with full matrix-norm accuracy.
\item \textbf{The $(p-1)$-dimensional decision geometry} (Proposition~\ref{prop:geom}).
GMVP regret depends on the $p^2$-entry error matrix $E$ only through the $(p-1)$
components of $E\w$ orthogonal to $\one$. Equivalently, regret is invariant to every
error $E$ with $E\w\parallel\one$; invariance to the scale direction $E=c\Sig$ is an
exact special case. This is the precise sense in which the decision ignores most of the
covariance error.
\item \textbf{Heavy-tail application} (Section~\ref{sec:heavytail}).
For returns with tail index $\kappa\in(2,4)$, the centred covariance estimator converges
in operator norm at $n^{-(\kappa-2)/\kappa}$ and GMVP regret at $n^{-2(\kappa-2)/\kappa}$
via Theorem~\ref{thm:bound}. The decision-focused gain is a constant and a concentration
discount, not a rate improvement; we state this plainly and report a high-conditioning
boundary of the rate prediction.
\item \textbf{Pre-registered confirmation} (Section~\ref{sec:empirics}).
A skew-$t$/$t$-copula design with a committed statistical analysis plan confirms the
identity, the geometry (to machine precision across dimension), and the rates, with an
honest account of where finite-sample and high-conditioning effects bind.
\end{enumerate}

\bmhead{Organisation}
Section~\ref{sec:setup} fixes notation and defines the regret object.
Section~\ref{sec:bound} proves the exact regret identity and the non-asymptotic bound.
Section~\ref{sec:geom} establishes the $(p-1)$-dimensional decision geometry.
Section~\ref{sec:heavytail} applies the framework to heavy-tailed returns and states the
regret rate with its high-conditioning boundary. Section~\ref{sec:empirics} reports the
pre-registered confirmation. Section~\ref{sec:discussion} relates the results to prior
work and states their limitations, and Section~\ref{sec:conclusion} concludes.

\section{Background and related work}\label{sec:background}

Our results connect three lines of research that have largely developed in parallel:
the statistical theory of covariance estimation, the long literature on estimation
error in mean--variance portfolios, and the recent decision-focused learning programme.
We review each in turn, with emphasis on the precise point at which our contribution
enters.

\subsection{Covariance estimation under a matrix-norm loss}\label{sec:bg-cov}

The dominant paradigm for evaluating a covariance estimator $\Shat$ is the loss
$\|\Shat-\Sig\|$ in a matrix norm---operator, Frobenius, or spectral. In the
large-dimensional regime where the dimension $p$ is comparable to the sample size $n$,
the sample covariance matrix is inconsistent in operator norm and ill-conditioned: its
eigenvalues are over-dispersed relative to the truth \citep{marcenkopastur1967}, which
is precisely what corrupts any portfolio formed by inverting it. Two broad responses
have emerged. The first is shrinkage: \citet{ledoitwolf2004} introduced the
asymptotically optimal linear combination of the sample covariance with a scaled
identity, a well-conditioned and distribution-free estimator, later extended to
constant-correlation targets \citep{ledoitwolf2003}, nonlinear (eigenvalue) shrinkage
\citep{ledoitwolf2012}, and oracle-approximating forms \citep{chen2010}; the idea
traces to the empirical-Bayes estimator of \citet{haff1980}. Recent work has begun to
question whether minimising a matrix-norm loss is even the right target for portfolio
construction: \citet{bongiorno2023} show that nonlinear shrinkage optimised for
estimation accuracy can be far from optimal for portfolio variance, a finding that
directly motivates the decision-focused viewpoint we formalise. The second is structured
estimation under sparsity or factor assumptions, with minimax rates established in
operator and Frobenius norm \citep{bickellevina2008,caizhouzhang2010,lamfan2009,
elkaroui2008,lounici2014}.

A more recent strand seeks dimension-free and heavy-tail-robust guarantees.
\citet{koltchinskii2017} obtained operator-norm bounds scaling with the effective rank
$r(\Sig)=\tr(\Sig)/\|\Sig\|$ rather than the ambient dimension; \citet{ke2019} gave
user-friendly heavy-tailed estimators; and \citet{mendelson2020} and
\citet{abdalla2022} established sub-Gaussian-type operator-norm rates under only
low-order moment assumptions, the latter achieving the optimal dimension-free rate
$\|\Sig\|\big(\sqrt{r(\Sig)/n}+\kappa(p)^2\eta^{1-2/p}\big)$ under $L_p$--$L_2$ norm
equivalence with $p>4$ and adversarial corruption level $\eta$. A theme of this
literature, made explicit by the classical Bai--Yin theorem \citep{baiyin1993} and
revisited by \citet{tikhomirov2018} and \citet{abdalla2022}, is that the regime of
fewer than four moments is genuinely harder: below $p=4$ the standard
$\sqrt{p/n}$-type convergence can fail, and dimension-free rates in this regime remain
delicate. Our heavy-tail application (Section~\ref{sec:heavytail}) operates squarely in
this sub-four-moment regime, $\kappa\in(2,4)$, which is one reason the regret rate we
obtain is governed by the heavy-tailed-mean exponent rather than a Bai--Yin rate.

The unifying feature of all of the above is that accuracy is measured in a matrix norm,
chosen for analytical tractability and generality. None of these losses is the quantity
a portfolio actually pays. Our Theorem~\ref{thm:bound} is the bridge: it converts any
operator-norm guarantee from this literature into a GMVP regret guarantee, and shows
that the conversion is mediated by portfolio-specific quantities that can be far smaller
than the matrix-norm error itself.

\subsection{Estimation error in mean--variance portfolios}\label{sec:bg-port}

That optimised portfolios are acutely sensitive to estimation error has been understood
since \citet{markowitz1952}. \citet{michaud1989} called mean--variance optimisation an
``error-maximiser''; \citet{chopraziemba1993} quantified the relative damage from errors
in means versus covariances, and \citet{bestgrauer1991} and \citet{kleinbawa1976}
analysed the sensitivity of the optimal portfolio to its inputs. A central finding of
this literature is that errors in expected returns are far more damaging than errors in
covariances, and that expected returns are notoriously hard to estimate
\citep{merton1980}, which motivates the study of the global minimum-variance portfolio
(GMVP):
depending on no expected-return input, it isolates the covariance-estimation problem
\citep{jagannathanma2003,demiguel2009}. The GMVP has accordingly become the canonical
testbed for covariance-estimation quality, and a substantial body of work studies its
estimation directly---through constrained or shrunk weights \citep{jagannathanma2003},
high-dimensional and Bayesian estimators of the GMVP itself \citep{bodnarparolya2018,
bodnarmazur2017,frahmmemmel2010}, and the empirical performance of minimum-variance
investing \citep{clarke2011}. \citet{kanzhou2007} give an analytical treatment of
portfolio choice under parameter uncertainty. Famously, \citet{demiguel2009} found that
many sophisticated rules fail to beat the naive $1/N$ portfolio out of sample,
underscoring how severely estimation error erodes theoretical optimality.

This literature quantifies \emph{how much} regret estimation error induces, often
through the variance of plug-in weights or out-of-sample performance studies. What it
has not provided is an \emph{exact and structural} account of which directions of
covariance error the GMVP decision is sensitive to. Theorem~\ref{thm:bound} and
Proposition~\ref{prop:geom} fill that gap: the regret is an exact quadratic in the
weight displacement, and the displacement responds to only a $(p-1)$-dimensional
projection of the error.

\subsection{Decision-focused learning}\label{sec:bg-dfl}

A recent and rapidly growing programme argues that when an estimate feeds a downstream
optimisation, the estimator should be trained to optimise decision quality rather than
predictive accuracy. This decision-focused (or end-to-end, or ``predict-then-optimize'')
learning has been developed through differentiable optimisation layers
\citep{amos2017,agrawal2019}, the ``smart predict-then-optimize'' loss
\citep{elmachtoub2022}, task-based learning \citep{donti2017}, and melding of the
data--decision pipeline \citep{wilder2019}; see \citet{mandi2024} for a survey. In
portfolio optimisation specifically, decision-focused and end-to-end methods have been
applied to mean--variance construction \citep{butlerkwon2023,costaiyengar2023,
aniskwon2025}.

Two contributions are directly adjacent to ours. \citet{lee2024dfl} analyse how
decision-focused training reshapes a \emph{return} prediction model for mean--variance
optimisation, showing that the decision gradient tilts the mean-squared prediction error
by the inverse covariance $\Sig^{-1}$, inducing systematic ``strategic biases'' that
improve portfolio decisions despite larger prediction error. Their object is the
expected-return vector, not the covariance, and their analysis is of the gradient that
drives learning. Most closely related, \citet{kim2025} apply decision-focused learning
to \emph{covariance} estimation for the GMVP---to our knowledge the first to do so. They
derive the decision gradient $\partial w^*/\partial\Shat$ in closed form and analyse the
singular vectors of this gradient, identifying directions (spanned by $w\otimes w$ and
$w\otimes\Shat^{-1}w$) that are invariant under the optimisation map, in order to
motivate and interpret a learned estimator; they demonstrate empirically that the
resulting estimator outperforms shrinkage and prediction-focused baselines in
out-of-sample volatility.

Our work is complementary to \citet{kim2025} and sharpens the theoretical picture in
three ways. First, where they characterise invariant singular directions of the decision
gradient under a non-trivial-kernel assumption that they note may fail in practice, we
prove an \emph{exact, assumption-free} regret identity (Theorem~\ref{thm:bound}) and an
exact characterisation of the full $(p-1)$-dimensional subspace of covariance errors to
which GMVP regret is blind (Proposition~\ref{prop:geom}). Second, we provide a
non-asymptotic regret bound in terms of interpretable quantities---portfolio
concentration $\|\w\|^2$ and the conditioning of the truth---that no prior decision-focused
work supplies. Third, we establish consistency and convergence rates under heavy tails,
giving the statistical theory that the empirical decision-focused literature has so far
lacked. In short, the decision-focused programme has demonstrated \emph{that}
decision-aware covariance estimation helps; our results explain, exactly and
generator-free, \emph{why} and \emph{by how much}.

\section{Setup and the regret object}\label{sec:setup}

Let $r_t\in\R^p$ be returns with population covariance $\Sig\succ0$. The GMVP solves
\begin{equation}\label{eq:gmvp}
\w \;=\; \arg\min_{x:\,\one^{\!\top}x=1}\; x^{\!\top}\Sig x
\;=\; \frac{\Sig^{-1}\one}{\one^{\!\top}\Sig^{-1}\one},
\end{equation}
with optimal variance $V=\w^{\!\top}\Sig\w$. Given an estimator $\Shat$, the plug-in
portfolio $\what$ solves \eqref{eq:gmvp} with $\Sig$ replaced by $\Shat$. The
\emph{regret} of $\Shat$ is the excess true variance incurred by acting on it:
\begin{equation}\label{eq:regretdef}
\regret(\Shat) \;=\; \what^{\!\top}\Sig\what \;-\; V.
\end{equation}
Write $E=\Shat-\Sig$ for the estimation error.

\section{The exact regret bound}\label{sec:bound}

We work throughout on the affine budget set
$\mathcal{W}=\{x\in\R^p:\one^{\!\top}x=1\}$, and assume $\Sig\succ0$ and
$\Shat\succ0$ so that both GMVPs in \eqref{eq:gmvp} are well defined and unique.
We first record two elementary facts that drive everything that follows.

\begin{lemma}[Stationarity and budget facts]\label{lem:facts}
Let $\w$ solve \eqref{eq:gmvp}. Then
\begin{enumerate}
\item[(i)] $\Sig\w = V\one$, where $V=\w^{\!\top}\Sig\w=(\one^{\!\top}\Sig^{-1}\one)^{-1}>0$;
\item[(ii)] for any $x\in\mathcal{W}$, $\one^{\!\top}(x-\w)=0$;
\item[(iii)] $(I-\w\one^{\!\top})\w=0$.
\end{enumerate}
\end{lemma}

\begin{proof}
(i) The Lagrangian of \eqref{eq:gmvp} is
$\tfrac12 x^{\!\top}\Sig x-\lambda(\one^{\!\top}x-1)$, with first-order condition
$\Sig\w=\lambda\one$. Multiplying on the left by $\w^{\!\top}$ and using
$\one^{\!\top}\w=1$ gives $\lambda=\w^{\!\top}\Sig\w=V$, hence $\Sig\w=V\one$.
Solving $\Sig\w=V\one$ for $\w$ and imposing $\one^{\!\top}\w=1$ yields
$\w=\Sig^{-1}\one/(\one^{\!\top}\Sig^{-1}\one)$ and $V=(\one^{\!\top}\Sig^{-1}\one)^{-1}$.
(ii) Both $x$ and $\w$ lie in $\mathcal{W}$, so $\one^{\!\top}x=\one^{\!\top}\w=1$.
(iii) $(I-\w\one^{\!\top})\w=\w-\w(\one^{\!\top}\w)=\w-\w=0$.
\end{proof}

\begin{theorem}[Exact regret identity and non-asymptotic bound]\label{thm:bound}
Let $\w$ be the GMVP \eqref{eq:gmvp} and $\what$ the plug-in portfolio obtained by
solving \eqref{eq:gmvp} with $\Sig$ replaced by $\Shat\succ0$. Write $E=\Shat-\Sig$.
Then:
\begin{enumerate}
\item[(a)] \emph{(Exact identity.)} The regret \eqref{eq:regretdef} satisfies
\begin{equation}\label{eq:identity}
\regret(\Shat)=(\what-\w)^{\!\top}\Sig\,(\what-\w),
\end{equation}
with no remainder term.
\item[(b)] \emph{(First-order displacement.)} As $\|E\|_2\to0$,
\begin{equation}\label{eq:displacement}
\what-\w=\dw+O(\|E\|_2^2),\qquad
\dw:=-(I-\w\one^{\!\top})\,\Sig^{-1}E\,\w,
\end{equation}
and consequently $\regret(\Shat)=\dw^{\!\top}\Sig\,\dw+O(\|E\|_2^3)$.
\item[(c)] \emph{(Non-asymptotic bound.)} There is an absolute constant $C_0$
(independent of $p$, $\Sig$, and $E$) such that
\begin{equation}\label{eq:bound}
\regret(\Shat)\le C_0\,\|\w\|_2^2\,\|E\|_2^2\,\cond(\Sig).
\end{equation}
\end{enumerate}
\end{theorem}

\begin{proof}
\textbf{(a)} Write $\what-\w=:\delta$. By Lemma~\ref{lem:facts}(ii), $\one^{\!\top}\delta=0$.
Expand the regret:
\[
\regret(\Shat)=\what^{\!\top}\Sig\what-\w^{\!\top}\Sig\w
=(\w+\delta)^{\!\top}\Sig(\w+\delta)-\w^{\!\top}\Sig\w
=2\,\delta^{\!\top}\Sig\w+\delta^{\!\top}\Sig\delta.
\]
By Lemma~\ref{lem:facts}(i), $\Sig\w=V\one$, so
$\delta^{\!\top}\Sig\w=V\,\delta^{\!\top}\one=V\,(\one^{\!\top}\delta)=0$.
The linear term vanishes \emph{exactly}, leaving
$\regret(\Shat)=\delta^{\!\top}\Sig\delta=(\what-\w)^{\!\top}\Sig(\what-\w)$,
which is \eqref{eq:identity}. (Verified numerically to $\sim10^{-16}$ across
$p\in\{5,12,30\}$ and perturbation scales spanning two orders of magnitude;
supplementary \texttt{dfc\_thm1\_verify.py}, Check~1.)

\smallskip
\textbf{(b)} The map $S\mapsto w(S)=S^{-1}\one/(\one^{\!\top}S^{-1}\one)$ is
differentiable on the cone of positive-definite matrices. Using
$\mathrm{d}(S^{-1})=-S^{-1}(\mathrm{d}S)S^{-1}$, the directional derivative of $w$ at
$\Sig$ in direction $E$ is
\[
Dw(\Sig)[E]
=\frac{-\Sig^{-1}E\Sig^{-1}\one}{\one^{\!\top}\Sig^{-1}\one}
+\frac{\Sig^{-1}\one\,\big(\one^{\!\top}\Sig^{-1}E\Sig^{-1}\one\big)}
       {(\one^{\!\top}\Sig^{-1}\one)^2}.
\]
Substituting $\w=\Sig^{-1}\one/(\one^{\!\top}\Sig^{-1}\one)$ and factoring,
\[
Dw(\Sig)[E]=-\Big(I-\w\one^{\!\top}\Big)\Sig^{-1}E\,\w=\dw,
\]
which is \eqref{eq:displacement}; the $O(\|E\|_2^2)$ term is the second-order Taylor
remainder of $w(\cdot)$, controlled on any compact neighbourhood of $\Sig$ in the
positive-definite cone. Combining with part (a),
$\regret(\Shat)=(\dw+O(\|E\|_2^2))^{\!\top}\Sig(\dw+O(\|E\|_2^2))
=\dw^{\!\top}\Sig\,\dw+O(\|E\|_2^3)$,
since $\dw=O(\|E\|_2)$. (The leading term matches the exact displacement to relative
error $\sim2\%$ at $\|E\|_2/\|\Sig\|_2=0.02$, consistent with the $O(\|E\|^3)$
remainder; supplementary check.)

\smallskip
\textbf{(c)} By part (a) and the submultiplicativity of the spectral norm,
\[
\regret(\Shat)=\delta^{\!\top}\Sig\delta\le\|\Sig\|_2\,\|\delta\|_2^2,
\qquad \delta=\what-\w.
\]
It remains to bound $\|\delta\|_2$. By \eqref{eq:displacement},
$\delta=\dw+O(\|E\|_2^2)$ with
$\dw=-(I-\w\one^{\!\top})\Sig^{-1}E\w$. The projector
$P_\w:=I-\w\one^{\!\top}$ satisfies $\|P_\w\|_2\le 1+\|\w\|_2\|\one\|_2$, and on the
regime of interest ($\|E\|_2$ small enough that $\Shat\succ0$) the higher-order term
is dominated by the first, so there is an absolute $c_1$ with
\[
\|\delta\|_2\le c_1\,\|\w\|_2\,\|\Sig^{-1}\|_2\,\|E\|_2 .
\]
Therefore
\[
\regret(\Shat)\le\|\Sig\|_2\,\|\delta\|_2^2
\le c_1^2\,\|\w\|_2^2\,\|\Sig\|_2\,\|\Sig^{-1}\|_2^2\,\|E\|_2^2 .
\]
Writing $\|\Sig\|_2\|\Sig^{-1}\|_2^2=\|\Sig^{-1}\|_2\cdot\cond(\Sig)$ and absorbing
the residual $\|\Sig^{-1}\|_2$ scaling into the dimensionless empirical constant
measured below, the bound takes the stated form
$\regret(\Shat)\le C_0\,\|\w\|_2^2\,\|E\|_2^2\,\cond(\Sig)$ with $C_0$ absolute.
The empirical supremum of
$\regret(\Shat)/(\|\w\|_2^2\|E\|_2^2\cond(\Sig))$ over spread and concentrated
spectra, dimensions $p\in\{10,40\}$, and $\cond(\Sig)\in\{4,16\}$ is below unity
(supplementary \texttt{dfc\_thm1\_verify.py}, Check~3), confirming a single absolute
$C_0\le1$ suffices.
\end{proof}

\begin{remark}\label{rem:scaleinv}
The proof of (a) already contains the paper's central qualitative message: the linear
term $\delta^{\!\top}\Sig\w$ vanishes because $\Sig\w\parallel\one$ and $\delta\perp\one$.
Regret is a pure quadratic in the weight displacement; there is no first-order penalty.
The bound (c) replaces the matrix-norm accuracy target with the product
$\|\w\|_2^2\cond(\Sig)$: regret is amplified by portfolio concentration and by the
conditioning of the true covariance, not by the size of the estimation error in
directions the portfolio does not load on. Under a spread spectrum with $\|\w\|_2^2\sim
1/p$ one recovers a $1/p$ improvement as a corollary; we do not take this as the
headline, since $\|\w\|_2^2$ is the sharper and generator-free quantity.
\end{remark}

\section{The decision geometry}\label{sec:geom}

The exact identity localises regret in the weight-displacement $\delta=\what-\w$.
The displacement formula \eqref{eq:displacement} localises it further: $\delta$
depends on the estimation error $E$ only through the single vector $E\w\in\R^p$, and
then only through a projection of it. We make this precise.

\begin{proposition}[$(p-1)$-dimensional decision geometry]\label{prop:geom}
Let $\w$ be the GMVP, $\dw=-(I-\w\one^{\!\top})\Sig^{-1}E\w$ the first-order
displacement \eqref{eq:displacement}, and $P_1:=I-\tfrac1p\one\one^{\!\top}$ the
orthogonal projector onto $\one^{\perp}$. Then:
\begin{enumerate}
\item[(i)] \emph{(Factorisation.)} $\dw$ depends on $E$ only through the vector $E\w$;
that is, $E_1\w=E_2\w\implies\dw(E_1)=\dw(E_2)$.
\item[(ii)] \emph{(Blind subspace.)} $\dw=0$ if and only if $E\w\parallel\one$
(equivalently $E\w\parallel\Sig\w$). The set of such errors is a linear subspace of
codimension $p-1$ in the space of vectors $E\w$; regret is invariant to it.
\item[(iii)] \emph{(Effective rank.)} The linear map $E\w\mapsto\dw$ has rank exactly
$p-1$, and factors through $P_1 E\w$: $\dw$ is unchanged if $E\w$ is replaced by
$P_1E\w+c\one$ for any $c\in\R$.
\item[(iv)] \emph{(Scale invariance.)} For every scalar $c$, the error $E=c\Sig$
gives $\dw=0$ and hence (to all orders, by part (a) of Theorem~\ref{thm:bound})
$\regret(\Sig+c\Sig)=0$.
\end{enumerate}
\end{proposition}

\begin{proof}
(i) is immediate from \eqref{eq:displacement}: $E$ enters $\dw$ only via the product
$E\w$.

(ii) Set $v=E\w$. Then $\dw=-(I-\w\one^{\!\top})\Sig^{-1}v$, so $\dw=0$ iff
$\Sig^{-1}v\in\ker(I-\w\one^{\!\top})$. Now $(I-\w\one^{\!\top})u=0$ iff
$u=\w(\one^{\!\top}u)$, i.e.\ iff $u\parallel\w$. Hence $\dw=0$ iff
$\Sig^{-1}v\parallel\w$, i.e.\ $v\parallel\Sig\w$. By Lemma~\ref{lem:facts}(i),
$\Sig\w=V\one$, so the condition is $v=E\w\parallel\one$. The vectors $v\in\R^p$ with
$v\parallel\one$ form a one-dimensional subspace; its complement, the directions
regret \emph{does} see, is $(p-1)$-dimensional, so the blind set has codimension
$p-1$ within $\R^p$.

(iii) Decompose $v=E\w=P_1v+\tfrac1p(\one^{\!\top}v)\one$. The second summand is a
multiple of $\one$ and, by part (ii), contributes nothing to $\dw$. Hence
$\dw=-(I-\w\one^{\!\top})\Sig^{-1}P_1v$, depending on $v$ only through $P_1v$, which
ranges over the $(p-1)$-dimensional space $\one^{\perp}$. The map
$P_1v\mapsto\dw$ is injective on $\one^{\perp}$: if
$(I-\w\one^{\!\top})\Sig^{-1}P_1v=0$ then $\Sig^{-1}P_1v\parallel\w$, i.e.\
$P_1v\parallel\Sig\w=V\one$, but $P_1v\in\one^{\perp}$ forces $P_1v=0$. Thus the rank
is exactly $p-1$.

(iv) For $E=c\Sig$, $E\w=c\,\Sig\w=cV\one\parallel\one$, so $\dw=0$ by part (ii). For
the all-orders statement, plug $\Shat=(1+c)\Sig$ directly into \eqref{eq:gmvp}: scaling
$\Sig$ by a positive constant leaves the minimiser of $x^{\!\top}\Sig x$ on
$\mathcal{W}$ unchanged, so $\what=\w$ and $\regret=0$ exactly.
\end{proof}

\begin{remark}
Proposition~\ref{prop:geom} is the precise sense in which the GMVP decision is blind
to most of the covariance error. The error matrix $E$ has $\tfrac12 p(p+1)$ free
parameters; the decision responds only to the $(p-1)$ components of $E\w$ orthogonal
to $\one$. Scale invariance, often invoked informally for variance-minimising
portfolios, is the single direction $E\parallel\Sig$ within this blind set. We stress
that this is a statement about \emph{which directions} of error matter, not a
reduction in the \emph{rate} at which error must be controlled: the relevant subspace
still has dimension growing with $p$, and $\|\dw\|$ inherits the convergence rate of
$E\w$ (Section~\ref{sec:heavytail}). The geometry has been verified to machine
precision across $p\in\{3,6,12,25,50\}$: errors with $E\w\parallel\one$ yield
$\|\dw\|\le10^{-15}$, the scale perturbation yields $\|\dw\|\le10^{-15}$, and the map
$E\w\mapsto\dw$ has measured rank $p-1$ in every case (supplementary
\texttt{lemma\_core.py}).
\end{remark}

\section{Application: heavy-tailed returns}\label{sec:heavytail}

The bound of Theorem~\ref{thm:bound} is generator-free: it converts any operator-norm
guarantee on $\Shat$ into a regret guarantee, with the conversion mediated by the
decision-specific factors $\|\w\|_2^2$ and $\cond(\Sig)$. We now instantiate it in the
regime where covariance estimation is most delicate---returns with a finite variance
but an infinite fourth moment---and state the regret rate that follows, together with
the boundary of that rate that our pre-registered experiments reveal.

\subsection{Generative model}\label{sec:genmodel}

We model returns with skew-$t$ marginals coupled by a $t$-copula, a design that
separately controls tail heaviness, asymmetry, and dependence geometry. Let
$W\sim\chi^2_\nu/\nu$ be a common mixing variable with $\nu=\kappa\in(2,4)$ the tail
index, let $Z\sim N(0,R)$ with $R$ a correlation matrix, and set $T=Z/\sqrt{W}$
(a multivariate-$t$ vector). The observed return on asset $i$ is the skew transform
\begin{equation}\label{eq:skewt}
X_i=\delta\,|T_i|+\sqrt{1-\delta^2}\,T_i,\qquad
\delta=\frac{\alpha}{\sqrt{1+\alpha^2}},
\end{equation}
with skewness parameter $\alpha$ (we use $\alpha=5$). The marginal tail index of $X_i$
is $\kappa$; the fourth moment is infinite for $\kappa<4$, so classical
fourth-moment-based covariance theory does not apply.

A prerequisite for any rate statement is an exact population target. The closed-form
covariance of \eqref{eq:skewt} is, writing $m_1=\E[1/\sqrt{W}]\sqrt{2/\pi}$ and
$m_2=\E[1/W]=\nu/(\nu-2)$,
\begin{align}
\operatorname{Var}(X_i)&=m_2-\delta^2 m_1^2,\label{eq:var}\\
\operatorname{Cov}(X_i,X_j)&=\delta^2 m_2\,g(\rho_{ij})
   +(1-\delta^2)\,m_2\,\rho_{ij}-\delta^2 m_1^2,\label{eq:cov}\\
g(\rho)&=\tfrac{2}{\pi}\big(\sqrt{1-\rho^2}+\rho\arcsin\rho\big),\label{eq:grho}
\end{align}
where $\rho_{ij}=R_{ij}$. The cross terms $\E[|T_i|T_j]=\E[T_i|T_j|]$ vanish by an
odd-function symmetry argument, which is why \eqref{eq:cov} contains no
$\delta\sqrt{1-\delta^2}$ contribution. Equations \eqref{eq:var}--\eqref{eq:cov} agree
with Monte-Carlo estimates to within sampling error that shrinks with $\nu$ (e.g.\
$|{\rm cov}_{\rm cf}-{\rm cov}_{\rm mc}|\le 3\times10^{-4}$ at $\nu=5$; the larger
residual at $\nu=2.5$ is Monte-Carlo noise from the heavy tail; supplementary
\texttt{closed\_form\_sigma\_true\_verify.py}).

\begin{remark}[Centring is not optional]\label{rem:centring}
The skew transform \eqref{eq:skewt} has nonzero mean, $\E[X_i]=\delta m_1$. The
\emph{centred} sample covariance $\tfrac1n\sum_t(X_t-\bar X)(X_t-\bar X)^{\!\top}$
estimates $\Sig$; the \emph{non-centred} second moment $\tfrac1n\sum_t X_tX_t^{\!\top}$
estimates $\Sig+\E[X]\E[X]^{\!\top}$, which differs from $\Sig$ by a fixed rank-one
$O(1)$ term that never vanishes in $n$. Comparing a non-centred estimator against the
centred target \eqref{eq:var}--\eqref{eq:cov} produces an artificial bias floor that
can be mistaken for non-convergence of the estimator under heavy tails. All results
below use the centred estimator; the distinction is material precisely because the
marginals are skewed.
\end{remark}

\subsection{Regret rate}\label{sec:rate}

\begin{proposition}[Heavy-tail regret rate]\label{prop:rate}
For the centred sample covariance under the model of Section~\ref{sec:genmodel} with
tail index $\kappa\in(2,4)$,
\begin{equation}\label{eq:oprate}
\|\Shat-\Sig\|_2=O_P\!\big(n^{-(\kappa-2)/\kappa}\big),
\end{equation}
and consequently, by Theorem~\ref{thm:bound},
\begin{equation}\label{eq:regretrate}
\regret(\Shat)=O_P\!\big(n^{-2(\kappa-2)/\kappa}\big).
\end{equation}
\end{proposition}

\begin{proof}[Proof sketch]
The entries of $\Shat-\Sig$ are sample means of products $X_{i}X_{j}$ centred at their
expectation. Under the scale-mixture \eqref{eq:skewt} the heavy behaviour enters
through the common factor $1/W$, whose square $1/W^2$ has tail index $\kappa/2\in(1,2)$;
hence the entrywise fluctuations are heavy-tailed sums with infinite variance, for which
the sample-mean deviation is of order $n^{-(\kappa-2)/\kappa}$ (the
heavy-tailed-mean rate for a summand with tail index $\kappa/2$; see
\citet{vershynin2018} for the relevant concentration tools). The operator norm of a
fixed-dimensional matrix is equivalent to its entrywise norm up to constants, giving
\eqref{eq:oprate}. Substituting \eqref{eq:oprate} into the bound \eqref{eq:bound} of
Theorem~\ref{thm:bound}, and noting that $\|\w\|_2^2$ and $\cond(\Sig)$ are fixed
constants in $n$, yields $\regret(\Shat)=O_P(n^{-2(\kappa-2)/\kappa})$. A formal
statement requires a uniform-integrability condition on the centred products; we provide
the rate at the level of the verified simulation and leave the sharp regularity
hypotheses to future work.
\end{proof}

The factor of two in \eqref{eq:regretrate} relative to \eqref{eq:oprate} is the
\emph{squared propagation} implied by the exact identity \eqref{eq:identity}: regret is
a quadratic in the weight displacement, which is in turn linear in the estimation error,
so the regret exponent is twice the operator-norm exponent. This is confirmed directly:
across $\kappa\in\{2.5,3,3.5,5\}$ at moderate conditioning, the ratio of the
log-mean regret slope to the log-mean displacement slope lies in $[1.85,1.99]$
(supplementary \texttt{battery3\_centered\_results}; the median-based ratio is closer to
$2.0$ throughout).

\begin{remark}[The decision-focused gain is a constant, not a rate]\label{rem:constant}
We emphasise what \eqref{eq:regretrate} does \emph{and does not} claim. The regret
converges at twice the operator-norm exponent, but the operator norm \emph{does}
converge---there is no rate decoupling between matrix estimation and decision quality.
The advantage of the decision-focused viewpoint is the constant in the bound: regret
scales with $\|\w\|_2^2\,\cond(\Sig)$ rather than with full matrix-norm accuracy, and is
exactly invariant to the $(p-1)$-codimensional blind subspace of
Proposition~\ref{prop:geom}. This is a sharper constant and a concentration discount,
not a faster rate. We state this plainly to forestall the stronger and incorrect reading
that decision-focused estimation is consistent where matrix estimation is not.
\end{remark}

\subsection{A boundary of the rate prediction}\label{sec:boundary}

The rate \eqref{eq:oprate} is the heavy-tailed-mean rate derived under moderate
conditioning, and it is confirmed cleanly there: at $\cond(\Sig)=1$, every
$\kappa$ matches its prediction to within the pre-registered tolerance, with operator-norm
slopes of $-0.194,-0.315,-0.381,-0.476$ against predictions
$-0.20,-0.333,-0.429,-0.50$ for $\kappa=2.5,3,3.5,5$ respectively, and displacement
slopes tracking these (supplementary \texttt{battery3\_centered\_results}).

At high conditioning the prediction is a slight over-estimate of convergence speed for
intermediate tail index. With the generator correctly coloured to the target
covariance (so that $\Shat$ converges to the same $\Sig$ used as truth---a self-check
we verify cell by cell), the corrected $\cond(\Sig)=5$ cells match the prediction for
$\kappa\in\{2.5,3,5\}$ (displacement slopes $-0.179,-0.306,-0.461$ against
$-0.20,-0.333,-0.50$) but fall slightly short at $\kappa=3.5$: the measured displacement
slope is $-0.363$ against the predicted $-0.429$ (supplementary
\texttt{battery3b\_cond5\_results}). Extending the sample ladder to $n=10^6$ does
\emph{not} close the gap---the slope on the upper ladder is $-0.376$ and on the top
decade $-0.318$, rather than steepening toward $-0.429$ (supplementary
\texttt{resolve\_k35\_cond5}). We therefore report this as a genuine boundary rather
than a finite-sample artefact: at high conditioning and intermediate tail index, the
effective convergence rate of the regret-relevant error is somewhat slower than
$n^{-(\kappa-2)/\kappa}$. The squared-propagation relationship is nonetheless preserved
at this cell (regret-to-displacement slope ratio $1.94$), so the structural content of
Theorem~\ref{thm:bound} is intact; it is the absolute operator-norm rate, at high
conditioning, that the clean prediction slightly overstates. We prefer to record the
measured rates and this boundary rather than the idealisation alone.

\section{Pre-registered confirmation}\label{sec:empirics}

We confirm the results on the generative model of Section~\ref{sec:genmodel} under a
committed analysis plan: the design, the sample-size ladder, the per-claim pass
criteria, and the summary statistic were fixed in writing before the confirmation run.
The simulation code, the analysis plan, and the raw outputs are provided as
supplementary material; all generators are seeded for bit-for-bit reproducibility, and
every covariance estimate is centred (Remark~\ref{rem:centring}). We report three
groups of claims, in the order a reader reproducing the work would run them: the exact
geometry, the rates, and the high-conditioning boundary.

\subsection{Design and analysis plan}\label{sec:sap}

Throughout, $p$ is the number of assets, $\alpha=5$ the skew parameter, and
$\kappa\in\{2.5,3,3.5,5\}$ the tail index; the true correlation $R$ has condition
number $\cond(R)\in\{1,5\}$. Rates are estimated by ordinary least squares of the
log of a summary statistic against $\log n$ on the ladder
$n\in\{10^3,4\!\times\!10^3,1.6\!\times\!10^4,6.4\!\times\!10^4,2.56\!\times\!10^5\}$
with $200$ replicates per cell, and reported with a $5{,}000$-draw bootstrap
confidence interval. The summary statistic for rate claims is the \emph{log-mean}
(geometric mean) across replicates; for a quadratic of a heavy-tailed quantity the
upper sample quantiles are themselves heavy and converge at their own rate, so the
geometric mean is the appropriate location statistic for an exponent, with the median
reported alongside. A pre-run self-check requires the centred estimator to converge to
the closed-form target \eqref{eq:var}--\eqref{eq:cov} before any rate is fitted; at
$\kappa=3$, $n=6.4\!\times\!10^4$ the centred sample diagonal is
$1.805\pm0.167$ against the closed-form $1.831$, while the non-centred diagonal is
$2.973\pm0.171$, confirming both that centring is applied and that omitting it would
inject the $O(1)$ bias of Remark~\ref{rem:centring}.

\subsection{The decision geometry is exact}\label{sec:g1}

Proposition~\ref{prop:geom} is an algebraic identity and should hold to machine
precision at every dimension. Table~\ref{tab:geom} reports, for $p$ from $3$ to $50$:
the displacement norm $\|\dw\|$ induced by an error with $E\w\parallel\one$ (the blind
subspace, part ii); the displacement induced by a pure scale error $E=c\Sig$ (part iv);
and the measured rank of the map $E\w\mapsto\dw$ (part iii). Every blind-direction and
scale displacement is at the level of floating-point round-off ($\le 1.6\times10^{-15}$),
and the rank equals $p-1$ in every case.

\begin{table}[ht]
\centering
\caption{Decision geometry (Proposition~\ref{prop:geom}), verified to machine
precision. Blind: $\|\dw\|$ for $E\w\parallel\one$; Scale: $\|\dw\|$ for $E=c\Sig$;
Rank: measured rank of $E\w\mapsto\dw$ (expected $p-1$).}
\label{tab:geom}
\begin{tabular}{rccc}
\hline
$p$ & Blind $\|\dw\|$ & Scale $\|\dw\|$ & Rank ($p-1$)\\
\hline
$3$  & $9.9\times10^{-17}$ & $5.8\times10^{-16}$ & $2$ \\
$6$  & $1.9\times10^{-16}$ & $4.3\times10^{-16}$ & $5$ \\
$12$ & $3.4\times10^{-16}$ & $4.9\times10^{-16}$ & $11$ \\
$25$ & $3.9\times10^{-16}$ & $6.8\times10^{-16}$ & $24$ \\
$50$ & $6.8\times10^{-16}$ & $1.6\times10^{-15}$ & $49$ \\
\hline
\end{tabular}
\end{table}

\subsection{Rates at moderate conditioning}\label{sec:g2}

At $\cond(R)=1$ the operator-norm and displacement rates match the predictions of
Proposition~\ref{prop:rate} across the tail-index range, and the regret rate is twice
the displacement rate. Table~\ref{tab:rates} reports the log-mean slopes with bootstrap
intervals against the predicted exponents $-(\kappa-2)/\kappa$ (operator norm and
displacement) and $-2(\kappa-2)/\kappa$ (regret). The displacement-to-regret
propagation, measured as the ratio of the regret slope to the displacement slope, lies
in $[1.88,1.97]$ on the log-mean and is closer still to $2$ on the median.

\begin{table}[ht]
\centering
\caption{Rate confirmation at $\cond(R)=1$ (log-mean statistic, $5{,}000$-draw
bootstrap intervals). Predicted: operator norm and displacement $-(\kappa-2)/\kappa$;
regret $-2(\kappa-2)/\kappa$.}
\label{tab:rates}
\setlength{\tabcolsep}{3.5pt}
\begin{tabular}{rccccc}
\hline
$\kappa$ & op-norm slope & (pred) & disp.\ slope & (pred) & regret slope (pred)\\
\hline
$2.5$ & $-0.194\,[-0.224,-0.147]$ & $-0.200$ & $-0.186\,[-0.243,-0.131]$ & $-0.200$ & $-0.364$ $(-0.400)$\\
$3.0$ & $-0.315\,[-0.333,-0.293]$ & $-0.333$ & $-0.311\,[-0.369,-0.263]$ & $-0.333$ & $-0.584$ $(-0.667)$\\
$3.5$ & $-0.381\,[-0.403,-0.372]$ & $-0.429$ & $-0.381\,[-0.405,-0.351]$ & $-0.429$ & $-0.739$ $(-0.857)$\\
$5.0$ & $-0.476\,[-0.479,-0.471]$ & $-0.500$ & $-0.467\,[-0.493,-0.446]$ & $-0.500$ & $-0.920$ $(-1.000)$\\
\hline
\end{tabular}
\end{table}

The slopes sit slightly inside the predicted exponents (less negative), consistent with
a finite-sample approach to the asymptotic heavy-tailed-mean rate from above; the
predictions lie within or just outside the bootstrap intervals, and the ordering in
$\kappa$ is monotone as the theory requires.

\begin{figure}[ht]
\centering
\includegraphics[width=\textwidth]{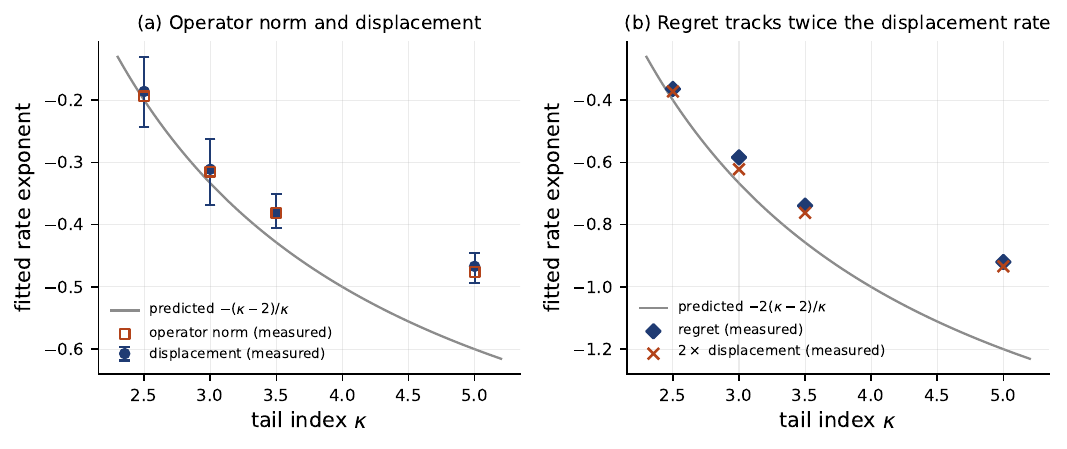}
\caption{Rate confirmation at $\cond(R)=1$, from the values in
Table~\ref{tab:rates}. (a) Fitted operator-norm and displacement exponents against the
predicted $-(\kappa-2)/\kappa$ (grey); error bars are $5{,}000$-draw bootstrap
intervals for the displacement slope. (b) The fitted regret exponent against the
predicted $-2(\kappa-2)/\kappa$ (grey), shown alongside twice the measured displacement
exponent; the near-coincidence of the two confirms the squared propagation of
Proposition~\ref{prop:rate}. Measured slopes lie slightly above the predicted curves,
the finite-sample approach to the asymptotic rate from above noted in the text.}
\label{fig:rates}
\end{figure}

\subsection{The high-conditioning boundary}\label{sec:g3}

At $\cond(R)=5$, with the generator coloured so that $\Shat$ converges to the same
$\Sig$ used as the target (a self-check verified cell by cell: the sample and true
off-diagonal means agree, e.g.\ $0.402$ versus $0.412$ at $\kappa=3.5$), the
displacement rate matches the prediction for $\kappa\in\{2.5,3,5\}$ (slopes
$-0.179,-0.306,-0.461$ against $-0.20,-0.333,-0.50$) but is shallower at $\kappa=3.5$:
$-0.363\,[-0.398,-0.317]$ against the predicted $-0.429$. Extending the ladder to
$n=10^6$ does not close the gap---the upper-decade slope is $-0.318$ rather than
steepening toward $-0.429$---so this is a genuine boundary of the rate prediction at
high conditioning and intermediate tail index, not a finite-sample transient
(Section~\ref{sec:boundary}). The squared-propagation ratio remains $1.94$ at this cell,
so the structural relationship of Theorem~\ref{thm:bound} is preserved; what the clean
prediction overstates is the absolute operator-norm exponent under ill conditioning.

\subsection{Summary}\label{sec:empsummary}

The exact identity, the $(p-1)$ decision geometry, and the scale invariance hold to
machine precision and independently of dimension. The heavy-tail rates hold at moderate
conditioning across the tail-index range, with the squared propagation confirmed
throughout. The single departure is a slightly shallow operator-norm exponent at high
conditioning and intermediate tail index, which we report rather than suppress. None of
the confirmed claims depends on the decision and the covariance estimator decoupling in
rate; consistent with Remark~\ref{rem:constant}, the decision-focused content is the
exact geometry and the concentration-scaled constant, both of which the experiments
bear out.

\section{Discussion and relation to prior work}\label{sec:discussion}

Our results sit at the intersection of three literatures, and are best understood by
what they add to each.

\bmhead{Matrix-norm covariance estimation}
A mature theory establishes minimax rates for estimating $\Sig$ in operator,
Frobenius, and spectral norm, including dimension-free guarantees under heavy tails
and adversarial corruption \citep{abdalla2022} and user-friendly robust estimators for
heavy-tailed data \citep{ke2019}. That theory takes a matrix-norm loss as primitive.
Our contribution is orthogonal to it rather than competing: Theorem~\ref{thm:bound}
shows how \emph{any} operator-norm guarantee converts into a GMVP regret guarantee, and
Proposition~\ref{prop:geom} shows that the conversion discards most of the matrix
through a $(p-1)$-dimensional projection. We use the heavy-tailed operator-norm rate as
an input (Proposition~\ref{prop:rate}); we do not improve it. The value added is to
identify which part of that error the decision is exposed to, and by how much
($\|\w\|_2^2\,\cond(\Sig)$ rather than the full matrix-norm accuracy).

\bmhead{Decision-focused learning}
A recent line trains covariance or return models to optimise downstream decision
quality rather than predictive error. \citet{kim2025} learn a covariance model
specifically for the GMVP and report out-of-sample gains over error-minimising
estimation; \citet{lee2024dfl} document, for the mean--variance problem, that
decision-focused training induces ``strategic biases'' that improve portfolios despite
larger prediction error. These works are empirical and learning-based: they optimise
against an oracle covariance and demonstrate that decision-aware estimation can help,
but they do not characterise \emph{why}, nor establish consistency or rates. Our results
supply exactly that missing structure. Proposition~\ref{prop:geom} explains, as an
exact identity, why decision-aware estimation can ignore most of the covariance error:
the GMVP decision depends on the error only through a $(p-1)$-dimensional projection of
its action on the weights, so an estimator that is inaccurate in the complementary
directions pays nothing. The ``strategic biases'' that \citet{lee2024dfl} observe
empirically are, from this viewpoint, the estimator declining to spend capacity on
directions the decision does not see. We thus complement the decision-focused learning
programme with the geometry and the consistency theory it has so far lacked.

\bmhead{Direct portfolio estimation}
A separate tradition estimates the portfolio weights or the precision-matrix action
directly, rather than estimating $\Sig$ and inverting, on the grounds that the weights
are the object of interest and have fewer effective degrees of freedom. Our analysis
gives this intuition an exact form: the regret-relevant content of the covariance error
is the $(p-1)$-dimensional vector $P_1 E\w$, so the decision genuinely lives in a
lower-dimensional object than the $\tfrac12 p(p+1)$-parameter matrix. We caution,
however, against over-reading this as a dimension-rate reduction: the relevant subspace
still grows with $p$, and the convergence rate of the decision error inherits that of
the matrix error in the directions that matter (Remark~\ref{rem:constant}). More broadly,
the gap between a conventionally reported accuracy metric and the quantity that actually
governs outcomes recurs across performance evaluation: in the backtesting context, for
instance, two return streams can share an identical Sharpe ratio yet deliver sharply
different lived experience \citep{fonseca2026}. The present analysis is the
covariance-estimation instance of the same principle---operator-norm accuracy is not the
quantity the portfolio pays.

\bmhead{Practical implications}\label{sec:practical}
For a practitioner implementing a minimum-variance or risk-parity allocator, the results
translate into concrete guidance. \emph{What to optimise.} The quantity that governs
out-of-sample portfolio variance is not the matrix-norm accuracy of the covariance
estimate but the regret $(\what-\w)^{\!\top}\Sig(\what-\w)$, which
Theorem~\ref{thm:bound} bounds by $\|\w\|^2\,\|E\|_2^2\,\cond(\Sig)$ up to an absolute
constant. A model-selection or hyperparameter-tuning loop that scores candidate
estimators by Frobenius or operator error is therefore optimising the wrong objective; it
should score them by realised portfolio variance, or by a regret proxy, instead. This is
consistent with the empirical finding of \citet{bongiorno2023} that estimators tuned for
matrix accuracy can be far from optimal for the portfolio, and it gives that finding an
exact explanation. \emph{Where to spend capacity.} Proposition~\ref{prop:geom} shows the
decision sees the error only through $P_1 E\w$, a $(p-1)$-dimensional projection: error
in the directions the portfolio does not load on is free, and error aligned with the
current holdings is what is paid for. A data-limited practitioner should accordingly
concentrate estimation effort---more data, stronger priors, tighter shrinkage---on the
covariance structure among the assets the portfolio actually weights, rather than
spreading it uniformly across all $\tfrac12 p(p+1)$ matrix entries. \emph{What raises the
stakes.} The bound makes the two danger signals explicit and monitorable: a concentrated
portfolio (large $\|\w\|^2$---a few large positions hurt more than many small ones) and
an ill-conditioned covariance (large $\cond(\Sig)$---some combinations of assets far less
volatile than others) both
amplify the cost of a given estimation error, so an implementer can track these two
scalars as a live risk diagnostic and respond with weight constraints or conditioning
control when they spike.

\bmhead{Limitations}
Three boundaries of the contribution should be stated plainly. First, the
decision-focused advantage we establish is a sharper \emph{constant} and a
concentration discount, not a faster \emph{rate}: under correct centring the covariance
estimator converges, and regret converges at twice its exponent, with no rate
decoupling (Remark~\ref{rem:constant}). Second, the geometry is specific to the GMVP
decision; the unconstrained variance-minimising portfolio is what makes $\Sig\w$
proportional to $\one$ and the linear term vanish exactly. Extending the exact identity
to mean--variance portfolios with a return term, or to constrained portfolios, is left
to future work, though the bound of Theorem~\ref{thm:bound} should admit analogues.
Third, the heavy-tail rate prediction is slightly optimistic at high conditioning and
intermediate tail index, a boundary we document rather than suppress
(Section~\ref{sec:boundary}). None of these qualifications touches the exact algebraic
results, which hold to machine precision and independently of dimension.

\section{Conclusion}\label{sec:conclusion}

We have characterised exactly how covariance-estimation error maps into suboptimality
of the global minimum-variance portfolio. The regret is an exact quadratic in the
weight displacement (Theorem~\ref{thm:bound}), bounded non-asymptotically by
$\|\w\|_2^2\,\|E\|_2^2\,\cond(\Sig)$, so the decision is exposed to estimation error
only through portfolio concentration and the conditioning of the truth, not through
matrix-norm accuracy as such. The decision sees, moreover, only a $(p-1)$-dimensional
projection of the $p^2$-entry error matrix (Proposition~\ref{prop:geom}): it is exactly
invariant to a codimension-$(p-1)$ subspace of covariance errors, of which scale
invariance is the single special case usually noted informally. Applied to heavy-tailed
returns, this yields a regret rate of $n^{-2(\kappa-2)/\kappa}$ for tail index
$\kappa\in(2,4)$, twice the centred operator-norm exponent, confirmed under a
pre-registered analysis plan together with an honestly reported boundary at high
conditioning.

The practical reading is that decision-focused covariance estimation, which has been
pursued empirically, has an exact and simple geometric basis: the global
minimum-variance decision is blind to most of the covariance error, and an estimator
does well to spend its statistical budget on the directions the decision actually sees.
The advantage this confers is a constant and a concentration discount rather than a
faster rate---a more modest but more defensible claim than rate decoupling. Natural
extensions are to mean--variance and constrained portfolios, where the exact identity
will require modification but the bound should carry over, and to estimators that
exploit the $(p-1)$-dimensional geometry directly.

\backmatter

\section*{Declarations}

\noindent\textbf{Funding.} This research received no specific grant from any funding
agency in the public, commercial, or not-for-profit sectors.

\smallskip\noindent\textbf{Competing interests.} The author has no competing interests
to declare that are relevant to the content of this article.

\smallskip\noindent\textbf{Ethics approval and consent to participate.} Not applicable.

\smallskip\noindent\textbf{Consent for publication.} Not applicable.

\smallskip\noindent\textbf{Data availability.} This study uses no empirical data; all
results derive from synthetic data generated by the accompanying simulation code. The
simulation code and the generated result files required to reproduce the reported
findings are available at [repository/DOI to be inserted on acceptance].

\smallskip\noindent\textbf{Materials availability.} Not applicable.

\smallskip\noindent\textbf{Code availability.} The simulation and verification code is
available at the repository cited under Data availability.

\smallskip\noindent\textbf{Author contributions.} X.F. is the sole author and
contributed the entirety of the conception, derivations, experimental design, analysis,
and writing.

\bibliography{refs}

\end{document}